\documentclass[letterpaper, 10 pt, conference]{ieeeconf}  

\IEEEoverridecommandlockouts                              

\usepackage{amsmath,amssymb,mathtools,bm}

\usepackage{graphicx}
\usepackage{booktabs,makecell,multirow}
\usepackage{subcaption} 

\usepackage{cite}
\usepackage{url}
\usepackage[hidelinks]{hyperref}

\usepackage[linesnumbered,ruled,vlined]{algorithm2e}
\usepackage{algorithmic}

\usepackage{xcolor}
\usepackage{graphicx}

\newcommand{\R}{{\mathbb{R}}}

\newcommand{\G}{\mathcal{G}}

\newcommand{\T}{\boldsymbol{T}}

\DeclareMathOperator*{\argmax}{arg\,max}
\DeclareMathOperator*{\argmin}{arg\,min}

\newlength{\dhatheight}

\newtheorem{lemma}{Lemma}

\newtheorem{proposition}{Proposition}

\title{\LARGE \bf
Multi-Agent Next-Best-View Optimization for Risk-Averse Planning
}

\author{
Amirhossein Mollaei Khass, Vivek Pandey, Guangyi Liu, Athanasios Cosse, Emrah Bayrak, Nader Motee%
\thanks{A.M. Khass, V. Pandey, A. Cosse, A. Bayrak and N. Motee are with the Department of Mechanical Engineering and Mechanics, Lehigh University. {\tt\small \{ammb23, vkp219, asc425, bayrak, motee\}@lehigh.edu}.}%
\thanks{G. Liu is with Amazon Robotics. {\tt\small gyliu@amazon.com}. This paper is independent of his position at Amazon and does not relate to his employment there.}%
}

\begin{document}

\maketitle
\thispagestyle{empty}
\pagestyle{empty}

\begin{abstract}
Multi-agent Next-Best-View (NBV) selection for safe path planning in uncertain and unknown environments requires informative, safety-aware, and efficient coordination. Centralized approaches rely on sharing raw sensor data or significant communication overhead, resulting in limited scalability. We propose a distributed, risk-aware multi-agent NBV framework in which each robot maintains a private local 3D Gaussian Splatting map and the team jointly maximizes expected information gain (EIG) restricted to masked zones along planned trajectories.
The resulting distributed objective is solved by Consensus ADMM (C-ADMM) over a communication graph, with each robot exchanging only candidate viewpoints, planned trajectory descriptors, and scalar EIG contributions. Collision risk along each trajectory is modeled via Average Value-at-Risk (AV@R) over the local 3DGS map and used both to shape the masking radius and to score planned paths. Experiments in Gibson environments at multiple team sizes show that the distributed formulation approaches the centralized baseline in mapping quality and trajectory safety while reducing communication by orders of magnitude.
\end{abstract}
\section{Introduction}
Autonomous navigation in unknown environments requires robots to simultaneously build a map while executing safe trajectories. Achieving this objective efficiently requires selecting informative viewpoints that reduce map uncertainty while ensuring safety. Recent advances in neural scene representations, particularly in 3D Gaussian Splatting (3DGS)~\cite{kerbl3Dgaussians} have enabled online, view-consistent mapping for robotic perception~\cite{chen2024splatnav,Matsuki2024_GaussianSplattingSLAM}. Uncertainty reduction is fundamental to active SLAM, where the objective is to select viewpoints to reduce ambiguity in scene reconstruction to enhance mapping~\cite{Jiang2024_AGSLAM} and localization~\cite{pandey2025efficient}. Building on these representations, active SLAM with NBV selection typically outperforms passive exploration in terms of mapping efficiency and information acquisition~\cite{Schwager2022_VISTA,Li2024_ActiveSplat}. However, safety must also be considered in active perception, as selecting informative viewpoints may lead robots toward uncertain regions of the environment where the risk of collision is higher~\cite{khass2026conflict}.

Neural scene representations including NeRF \cite{mildenhall2020nerf} and 3DGS~\cite{kerbl3Dgaussians} representations have  been extended to multi-robot settings ~\cite{yu2025_hammer,yugay2025_magic_slam,zhao2024_udon}. Among these approaches,  3DGS is especially appealing for robotic perception due to its faster optimization and more stable real-time rendering compared with NeRF. These properties make 3DGS well suited for online perception and active view planning in collaborative robotic teams. 
Despite these advances, fundamental challenges remain. Existing multi-robot neural mapping systems typically rely on centralized coordination, large map transmissions, or global alignment procedures, which become costly and limit scalability and deployment in communication-constrained environments.
Effective communication requires reasoning about how local observations influence the map uncertainty relevant to the agents' planned trajectory. 

To address these challenges, we propose a Distributed Multi-Robot Information-Driven and Risk-Aware NBV Optimization framework, in which the team of robots coordinates viewpoint selection in order to maximize mutual information with an efficient communication method. Unlike centralized methods, our formulation avoids constructing a global map or sharing raw data. 
In addition, each robot evaluates the safety of its future trajectory by considering the risk of collision via average value at risk (AVaR) assessment. This enables robots to balance information gathering with safety during exploration. By focusing information gain evaluation along planned trajectories,  the proposed distributed NBV optimization eliminates the need to compute information gain over a large global 3DGS map.

\textbf{Contributions.} This work makes the following contributions:
\begin{itemize}
    \item A distributed multi-agent Next-Best-View optimization framework over local 3DGS maps that maximizes team information gain while robots plan and execute risk-averse safe trajectories.
    \item A C-ADMM solver with communication payload independent of map size
    \item A risk-aware masked-zone construction using AV@R to acquire information for future trajectory
    \item Gibson environment evaluation across team sizes in centralized, limited-sharing, and single-agent baselines

\end{itemize}

\section{Related Work}
\textbf{Next Best View Selection for Uncertainty Reduction.} 
Recent work addresses uncertainty-aware viewpoint selection for NeRF and 3DGS map representations, with the goal of improving mapping quality~\cite{zhan2022activemap,ong2024unified3dgs,Schwager2022_VISTA}.
Recent work has also explored safety aware viewpoint selection as \cite{liu2024riskaware} couple Fisher based NBV decisions with collision risk modeling, while~\cite{khass2026conflict,Khass2025_NBV_RiskAverse} considers risk in decision making and develops risk averse safe path planning by optimizing NBV to maximize information gain.
Our work extends Fisher based NBV optimization to a distributed multi-robot setting, enabling coordinated viewpoint selection without centralized map fusion.
In this work, we build on MonoGS~\cite{Matsuki2024_GaussianSplattingSLAM} as the underlying scene representation for active SLAM.


\textbf{Multi-Agent Radiance-Field Mapping.} Recent work has extended neural scene representations to multi-robot settings. In this direction, RAMEN~\cite{zhao2025ramen} introduces real-time asynchronous multi-robot neural implicit mapping, enabling agents to update a shared implicit model without strict synchronization. MAGiC-SLAM~\cite{yugay2025_magic_slam} proposes a centralized multi-agent 3DGS SLAM framework that merges individual robot maps into a single global representation, improving tracking and loop closure performance. HAMMER~\cite{yu2025_hammer} introduces a heterogeneous multi-robot 3DGS system capable of continual online mapping with semantic capabilities. SIREN~\cite{Shorinwa2025_SIREN} addresses multi-robot 3DGS map through a semantic, initialization free alignment method, improving robustness when robots construct disjoint splat maps from diverse viewpoints.
In this work, we aim to enable distributed multi-robot NBV selection using lightweight information exchange, allowing each robot to plan informative viewpoints from its local 3DGS map. 
\section{Preliminaries}
This section establishes the problem formulation for our multi-robot active NBV selection framework. We first describe the scene representation based on 3DGS~\ref{subsec:3dgs_representation}. We then introduce the Expected Information Gain (EIG) objective used to evaluate candidate viewpoints~\ref{subsec:eig_formulation}, followed by the risk metric used to quantify the safety of future trajectories~\ref{subsec:risk_metric}. Finally, we formalize a risk-aware scoring function that prioritizes areas whose uncertainty most affects future navigation decisions.
\subsection{Scene Representation via 3D Gaussian Splats}
\label{subsec:3dgs_representation}
We represent the environment using the differentiable 3D Gaussian Splatting (3DGS) model~\cite{kerbl3Dgaussians}. 
Each scene element is an anisotropic Gaussian splat $g = (\mu, \Sigma, o, c)$ with mean-value $\mu \in \mathbb{R}^3$ and covariance 
$\Sigma = R S^2 R^\top \in \mathbb{S}^3_{++}$, where $R \in \mathrm{SO}(3)$ is the  orientation matrix, $S\!\in\!\mathbb{R}^3_{+}$ specifies axis-aligned scales, $o$ is the opacity, and $c$ is a low-dimensional color descriptor. Given a camera pose $T\!\in\!\mathrm{SE}(3)$, each  Gaussian splat is projected onto the image plane using a first-order linearization of the camera model
\[
\mu_I = \pi(T^{-1}\mu), 
\qquad
\Sigma_I = J \Sigma J^\top ,
\]
where $\pi(\cdot)$ is the projection map and $J$ is the Jacobian of the local warp.  
The rendered pixel color follows the standard alpha-compositing rule:
\begin{equation}   
    C_p = \sum_{i=1}^N c_i \,\rho_i \prod_{j<i}(1-\rho_j), 
\end{equation}
with blending weights $\rho_i$ induced by the 2D projected Gaussians.
Thus, training uses the volumetric 3DGS rendering loss:
\begin{equation}
\mathcal{L}
= \sum_p 
\Big(
\mathcal{L}_1(C_p,\hat{C}_p)
+
\psi\,\mathcal{L}_1(D_p,\hat{D}_p)
\Big),
\label{mapping}
\end{equation}
where $C_p$ and $D_p$ are the ground truth RGB and depth, $\hat{C}$ and $\hat{D}$ are the rendered predictions, $\mathcal{L}_1$ is the $L_1$ loss, and $\psi \in [0,1)$ balances depth relative to RGB.
\begin{figure*}[t]
\vspace{0.25cm}
    \centering
    \includegraphics[width=\textwidth,trim={1.8cm 6.5cm 3.8cm 0.7cm},
    clip]{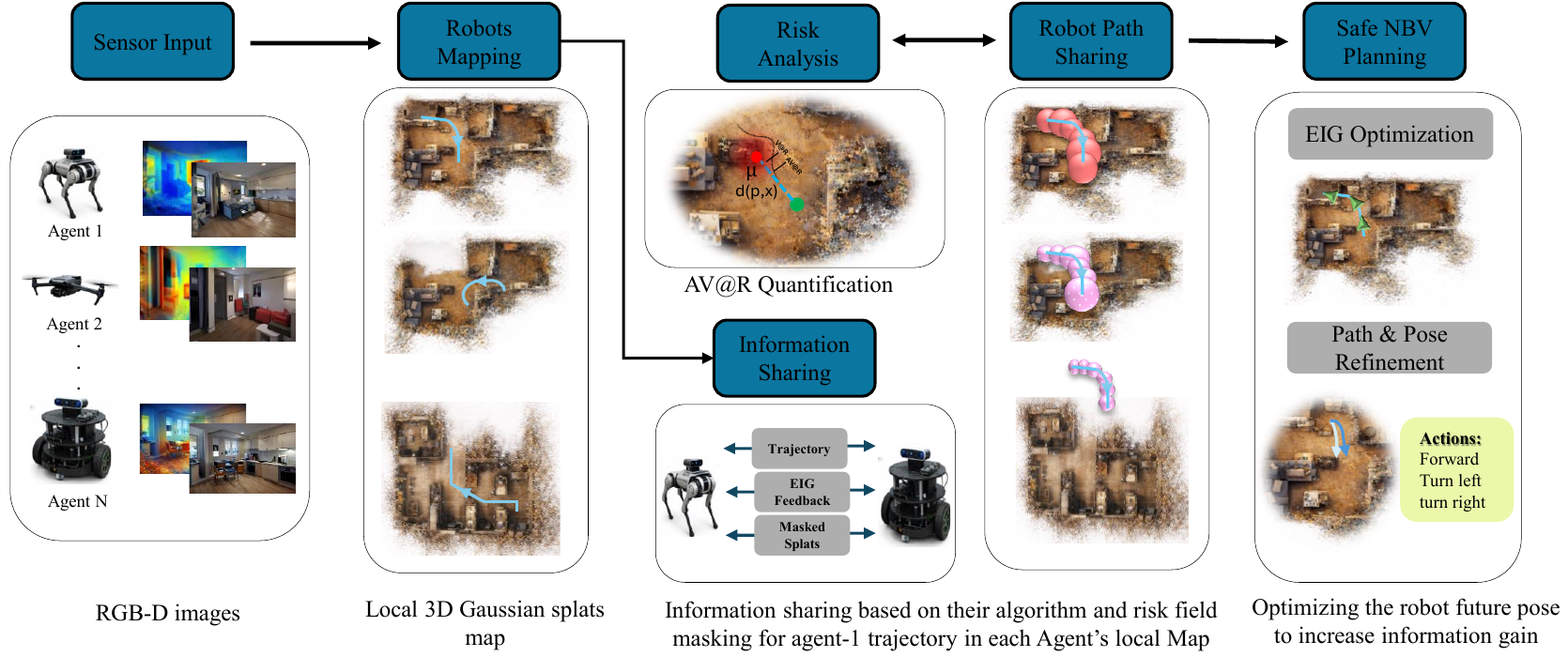}
    \caption{\textbf{System overview} of the distributed, information-driven, and risk-aware multi-robot NBV framework. In a team of $R$ agents, each agent constructs a local 3DGS map from RGB-D observations. Agents exchange future trajectory and EIG scores over masked regions along trajectories in their local maps. Viewpoints are selected to maximize risk-aware information gain, after which each robot executes the corresponding actions.}
\label{fig:system_overview}
    \vspace{-0.5cm}
\end{figure*}
\subsection{Expected Information Gain}
\label{subsec:eig_formulation}

To quantify the informativeness of a future viewpoint, we employ an information theoretic formulation based on entropy reduction \cite{kirsch2022unifying}.


Let $\omega$ denote the parameter vector describing the 3DGS scene representation $\Theta$.  
Following a second-order approximation of the likelihood around the current estimate
$\omega^\ast$, the {observed information} associated with a measurement
$(T_i, Y_i)$ is defined as the negative Hessian of the log-likelihood:
\begin{equation}
    \mathcal{H}[Y_i \mid T_i, \omega^\ast] 
    = 
    -\,\nabla^2_{\omega}\log p(Y_i \mid T_i, \omega)\big|_{\omega=\omega^\ast}.
    \label{eq:obs_info_single}
\end{equation}


In neural and volumetric rendering, each observation $(T,Y)$ is generated by a differentiable
function $f(T,\omega)$ representing the rendering model under robot pose $T$. Minimizing the negative
log-likelihood,
\begin{equation}
    -\log p(Y \mid T,\omega) = \| Y - f(T,\omega) \|^2,
\end{equation}

For selecting the next robot pose $T \in \mathrm{SE}(3)$ without explicitly rendering the corresponding observation, the expected information gain with respect to the current parameter estimate $\omega^\ast$ and candidate acquisition pose $(T^{\mathrm{acq}}, Y^{\mathrm{acq}})$ can be approximated using the Fisher information matrix as\cite{kirsch2022unifying, jiang2024Fisherrf}:
\begin{equation}
    \mathrm{EIG}(T^{\mathrm{acq}})
    =
    \mathrm{tr}\!\left(
        \mathcal{H}[Y^{\mathrm{acq}} \mid T^{\mathrm{acq}}, \omega^\ast]\,
        \mathcal{H}_{\text{prior}}[\omega^\ast]^{-1}
    \right),
    \label{eq:Fisher_bounds}
\end{equation}
which quantifies the expected reduction in posterior uncertainty contributed by viewpoint $T^{\mathrm{acq}}$.
The prior information matrix $\mathcal{H}_{\text{prior}}[\omega^\ast]^{-1}$ is obtained by accumulating the Hessians of the model parameters across all previously observed views before inversion. Moreover, the Hessian matrix is formed using the diagonal approximation in FisherRF~\cite{jiang2024Fisherrf} \[\mathcal{H}[Y \mid T, \omega^\ast] \simeq \mathrm{diag}( \nabla_\omega f(T,\omega)^\top\nabla_\omega f(T,\omega)) + \lambda I.\]

\subsection{Risk of Collision in 3DGS}
\label{subsec:risk_metric}
In order to quantify the robot's safe planning in an uncertain 3DGS map, we adopt the Average Value-at-Risk (AV@R) as a risk measure to quantify collision risk. For a random variable $ y : \Omega \!\to\! \mathbb{R}$ on probability space
$(\Omega,\mathcal{F},\mathbb{P})$, AV@R is defined~\cite{sarykalin2008varcvar} as
\begin{equation}
    \mathrm{AV@R}_{\varepsilon}(y)
    :=
    \mathbb{E}\!\left[
        y \;\middle|\; y < \mathrm{V@R}_{\varepsilon}(y)
    \right]
\end{equation}
where
\begin{equation}
    \mathrm{V@R}_{\varepsilon}(y)
    :=
    \inf\!\left\{
        z \in \mathbb{R} \;\middle|\; \mathbb{P}(y < z) > \varepsilon
    \right\}
\end{equation}
and $\varepsilon \in (0,1)$ specifies the risk level. Risk of collision for a robot at position $p \in \mathbb{R}^3$ to a Gaussian splat $g_s \sim \mathcal{N}(\mu_s, \sigma_s^2 I_3)$ with signed-distance variable
$d(p,g_s) \sim \mathcal{N}(\lVert \mu_s - p \rVert_2,\, \sigma_s^2)$,
can be expressed as
\begin{equation}
    \mathrm{AV@R}_{\varepsilon}(d(p,g_s))
    =
    \lVert p-\mu_s \rVert_2
    \;-\;
    \sigma_s \,
    \frac{1}{\varepsilon \sqrt{2\pi}}
    \exp\!\left( -\iota^2 \right)
    \label{eq:avar_gaussian}
\end{equation}
where 
\[
    \iota = \operatorname{erf}^{-1}(2\varepsilon - 1).
\]

To assess robot safety for planning purposes, we define a risk field function $\alpha: \R^3 \rightarrow \R$ by considering the worst-case collision risk over the 3DGS map $\mathcal{G}$ by 
\begin{equation}
    \alpha(p)
    :=
    \min_{g_s \in \mathcal{G}}
    \mathrm{AV@R}_{\varepsilon}\!\big( d(p, g_s) \big),
    \label{eq:risk_min_avar}
\end{equation} 

\subsection{Risk-Aware Masked Zones}
\label{subsec:NBV_optimization}
The unconstrained NBV problem considers finding an optimal pose by solving an optimization problem over the entire known map. This will reduce the overall uncertainty over the entire map. However, one can impose constraints and select NBVs to reduce uncertainty only along the robot's future trajectory \cite{Khass2025_NBV_RiskAverse} by constructing a risk-aware masked zone along the trajectory of the robot to identify high-risk  Gaussian splats.

Let the current planned trajectory of robot $i$ be the sequence of waypoints
$\mathcal{P}_i = \{p_1, p_2, \ldots, p_K\}$ with $p_k \in \mathbb{R}^3$, each
waypoint is assigned a masking radius
\begin{equation}
    r_{\mathrm{mask}}(p_k)
    = r_\text{min} +
    \beta_1 e^{-\beta_2 \alpha(p_k)},
    \qquad
    \beta_1, \beta_2 > 0
    \label{eq:masking_rad}
\end{equation}
where $\alpha(p_k)$ in Eq.~\eqref{eq:risk_min_avar} denotes the risk value at $p_k$.
For a given map $\G$, the masked zone of planned trajectory $\mathcal{P}$ is defined by
\begin{equation} 
    \Pi_{\mathcal{P}_i}\big|_{\G}
    =\Big\{g_{s} \in \G ~\Big| ~ \mu_s \in  \bigcup_{p_k \in \mathcal{P}_i}
    \mathcal{B}\!\left(p_k,\, r_{\mathrm{mask}}(p_k)\right)\Big\}
    \label{eq:masked_region_G}
\end{equation}
with Gaussian splat $g_s = (\mu_s, \Sigma_s, \alpha_s, c_s)$
\begin{equation}
    \Pi_{\mathcal{P}_i}
    =
    \bigcup_{p_k \in \mathcal{P}_i}
    \mathcal{B}\!\left(p_k,\, r_{\mathrm{mask}}(p_k)\right)
    \label{eq:masked_region}
\end{equation}

which identifies the region most relevant to risk-aware safety margin.

To construct a risk-aware information gain and prevent redundant multi-view acquisition, we use a proximity weighting function to prioritize nearby Gaussian splats $g_i$ in $\Pi_{\mathcal{P}}$ for safe navigation as~\cite{Khass2025_NBV_RiskAverse}:
\begin{equation}\label{eqn:V_function}
    v(g_s) = \gamma_1 \hspace{0.05cm} e^{-\gamma_2 \left\|p  - \mu_s\right\|_g},
\end{equation}
where $\lVert\cdot\rVert_g$ is the graph (grid-based) distance, and $p$ is the robot waypoint position. Therefore, by incorporating Fisher Information, the weighted Hessian approximation is defined:
\begin{equation}
H’'[T|\omega^\ast] := [\bar{V}\mathcal{H}[Y \mid T, \omega^\ast]],
\end{equation}
where $\bar{V} = \mathrm{diag}(V_1, \dots, V_{|\mathcal{G}|})$, with $V_s = v(g_s)\cdot I_{|\omega_s|}$ weighting the contribution of each Gaussian splat. Here, $\bar{\omega}_s$ is the parameter vector for splat $g_s$.
Finally, the Expected Information Gain (EIG) is given by:
\begingroup
\setlength{\abovedisplayskip}{3pt}
 \setlength{\belowdisplayskip}{3pt}
 \setlength{\abovedisplayshortskip}{2pt}
 \setlength{\belowdisplayshortskip}{2pt}
\begin{equation}
\mathcal{I}(T; \Pi_{\mathcal{P}}; \mathcal{G}) =
\mathrm{tr}\!\left(
H''[ T, \omega^\ast]\;
H''[\omega^\ast]^{-1}_{\mathrm{prior}}\right),
    \label{eq:weighted_eig}
\end{equation}
\endgroup
 over the local 3DGS map $\mathcal{G}$ and relative to the Gaussian splats in $\Pi_\mathcal{P}$~\cite{khass2026conflict,Khass2025_NBV_RiskAverse}. In this formulation, nearby trajectory-relevant Gaussian splats carry more information due to the risk of collision. 
\vspace{-0.15cm}
\section{Problem Statement}
\label{sec:method}
We consider a multi-agent scenario in which a team of $N$ robots collaboratively navigate a shared, partially known environment while jointly improving its 3D Gaussian Splatting (3DGS) reconstruction. Robots are indexed by $\mathcal{V} := \{1,\dots,N\}$, and we assume each robot localizes itself in a common world frame so that all local maps are expressed in global coordinates. Each robot $i \in \mathcal{V}$ maintains a private local 3DGS map $\mathcal{G}_i$, built and updated online from its onboard RGB-D stream. The 3DGS parameter vector of $\mathcal{G}_i$ is denoted $w_i$, and includes the means, covariances, opacities, and color descriptors of all Gaussian primitives owned by robot $i$. The centralized map of the entire
environment is then defined as the disjoint union of the robots'
local maps
\begin{equation}
  \mathcal{G}_\text{cent} \;=\; \bigsqcup_{i=1}^{N} \mathcal{G}_i.
  \label{eq:Gcent}
\end{equation}
The centralized multi-agent NBV problem aims to jointly optimize the viewpoints of all robots in order to maximize the total EIG over the shared centralized map $\mathcal{G}_{\mathrm{cent}}$. 
For each robot $i \in \mathcal{V}$, let $T_i \in SE(3)$ denote its candidate viewpoint and
$\mathcal{P}_i$ denote its planned trajectory. To restrict information evaluation to
regions relevant to robot motion, robot $i$ constructs a masked zone $\Pi_{\mathcal{P}_i}|_{\mathcal{G}_\text{cent}}$ along $\mathcal{P}_i$.
The team-wide masked zone over the centralized map is
\begin{equation}
  \Pi_\mathcal{P}\big|_{\mathcal{G}_\text{cent}}
  \;:=\;
  \bigcup_{i=1}^{N} \Pi_{\mathcal{P}_i}\big|_{\mathcal{G}_\text{cent}},
  \label{eq:team_mask}
\end{equation}
collects all trajectory-relevant Gaussian primitives in $\mathcal{G}_\text{cent}$.

Let $\boldsymbol{T} := (T_1, \dots, T_N)$ denote the team viewpoint vector.
The centralized multi-agent NBV problem jointly optimize $\boldsymbol{T}$ to maximize the total expected information gain (EIG) over the trajectory relevant region of the centralized map:
\begin{equation}
  \boldsymbol{T}^{*}
  \;=\;\underset{\{T_i\}_{i=1}^{N}}{\arg\textrm{maximize}}\;
  \;\sum_{i=1}^{N}
  \mathcal{I}\!\left(T_i;\;
\Pi_{\mathcal{P}_i}\big|_{\mathcal{G}_\text{cent}};\;
    \mathcal{G}_\text{cent}\right).
  \label{eq:cent_obj}
\end{equation}

The centralized formulation requires robots to transmit their local maps to a central unit to construct the global map $\mathcal{G}_{\mathrm{cent}}$, which results in costly communication. We develop a distributed formulation in which each robot evaluates EIG only on its own local map and exchanges compact viewpoint-related summaries rather than raw map representations.
\begin{figure}[t]
        \vspace{0.25cm}
    \centering
    \begin{subfigure}[b]{0.48\linewidth}
\centering\includegraphics[width=1\linewidth,,height=0.17\textheight]{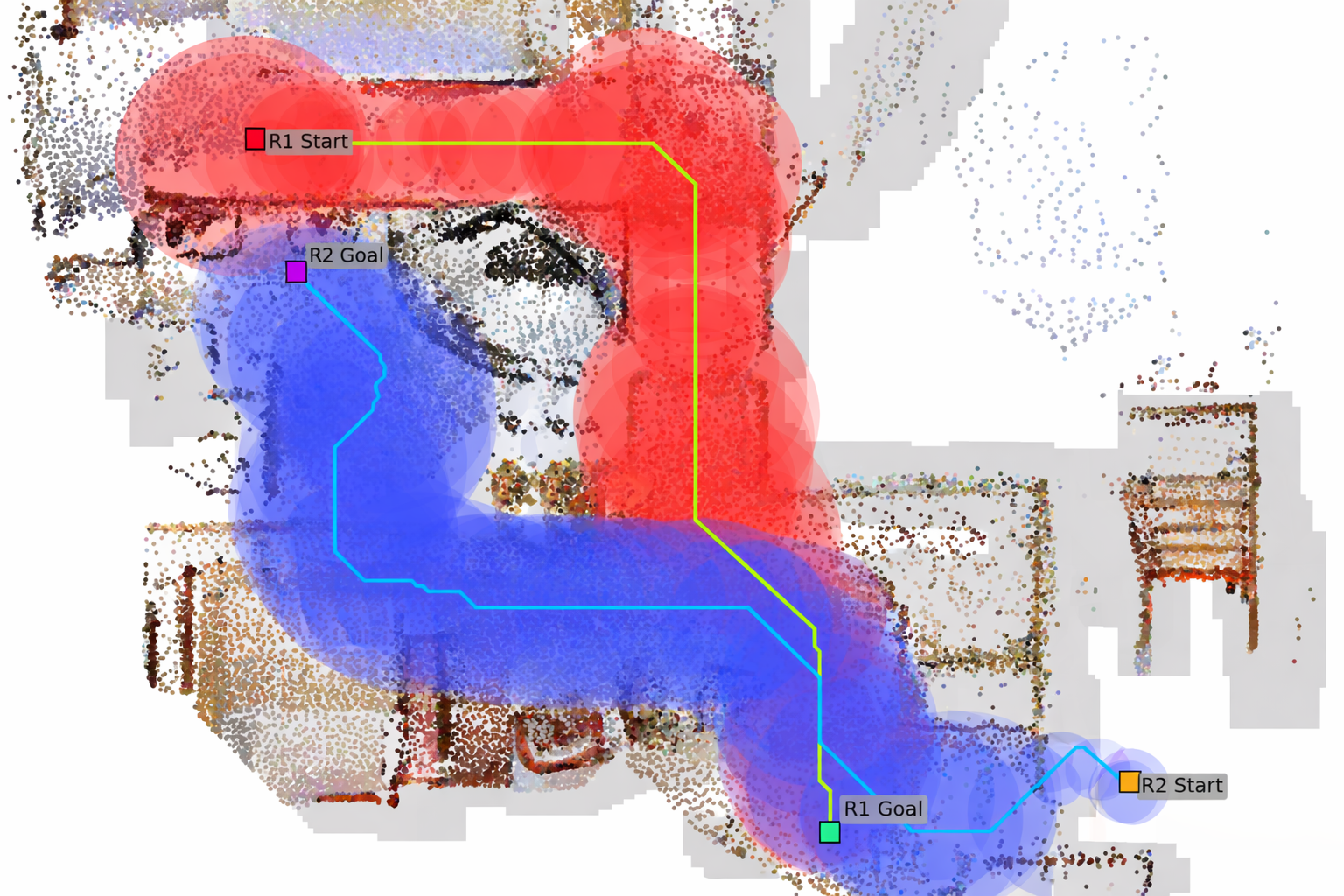}
        \caption{}
        \label{fig:swormnille_recon}
    \end{subfigure}
    \hfill
    \begin{subfigure}[b]{0.38\linewidth}
        \centering
        \includegraphics[width=1\linewidth,,height=0.17\textheight]{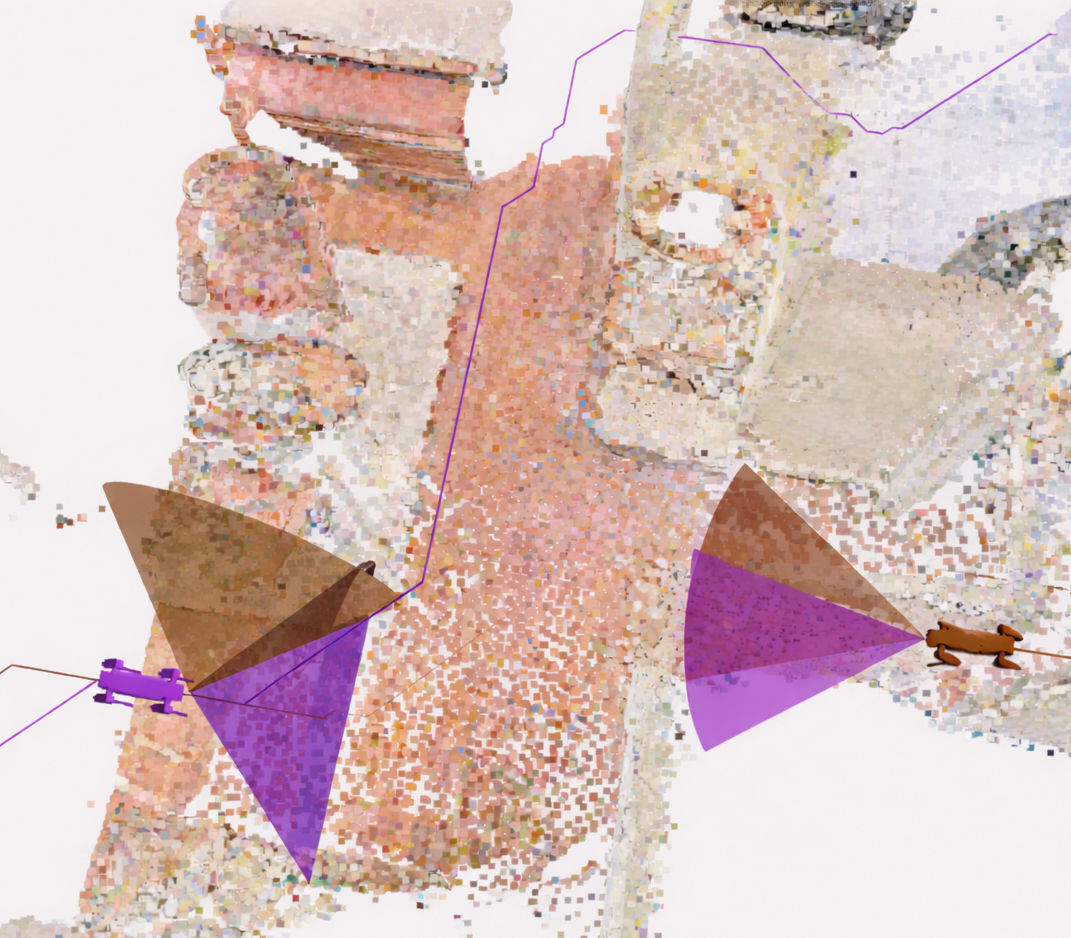}
        \caption{}
    \end{subfigure}
    \caption{(a) Planned safe trajectory with the corresponding mask. (b) Each robot proposes an informative view for itself and for its teammate from its local 3DGS map.}
    \vspace{-0.35cm}
\end{figure}

\section{Distributed NBV Selection}
\subsection{Formulation of Distributed Objective Function}
To formulate the distributed NBV optimization problem, we first decompose the centralized objective in~\eqref{eq:cent_obj} so that each robot evaluates information gain using only local maps.

\begin{proposition}[Block-additive decomposition of EIG]
\label{prop:localsum}
Under the diagonal Fisher information approximation, the observed
information matrix $\mathbf{H}''[Y \mid T, w^*]$ is block-diagonal with respect to the splat parameter induced by the disjoint union in~\eqref{eq:Gcent}. Therefore, the centralized EIG of robot $i$
decomposes across local map blocks as
\begin{equation}
  \mathcal{I}\!\left(T_i;\;
    \Pi_{\mathcal{P}_i}\big|_{\mathcal{G}_\text{cent}};\;
    \mathcal{G}_\text{cent}\right)
  \;=\;
  \sum_{j=1}^{N}
  \mathcal{I}\!\left(T_i;\;
\Pi_{\mathcal{P}_i}\big|_{\mathcal{G}_\text{cent}};\;
    \mathcal{G}_j\right).
  \label{eq:block_decomp}
\end{equation}
\end{proposition}
Here $  \mathcal{I}\!\left(T_i;\;
\Pi_{\mathcal{P}_i}\big|_{\mathcal{G}_\text{cent}};\;
\mathcal{G}_j\right)$ denotes the EIG contribution of robot $j$'s parameter block evaluated using the centralized map masking zone.

\begin{proof}
Under the diagonal approximation in~\eqref{eq:Fisher_bounds},
the Fisher information matrix is additive across independent parameter blocks~\cite{Khass2025_NBV_RiskAverse,jiang2024Fisherrf}. Since $\mathcal{G}_\text{cent}$ is a disjoint union, the parameter blocks $\{w_j\}_{j=1}^{N}$ are independent, and the trace based EIG over
$\mathcal{G}_\text{cent}$ decomposes into the sum of contributions from each local map $\mathcal{G}_j$. 
\end{proof}

The decomposition in~\eqref{eq:block_decomp} is structural. The $j$-th term represents the contribution of robot $j$'s private
local map to the centralized EIG of robot $i$'s candidate viewpoint. This is the key property that enables distributed information evaluation, since each map owner can compute its own contribution without transmitting raw observations or 3DGS
parameters.
\begin{proposition}[Local mask conservatism]
\label{prop:mask}
For any robot $i$ with planned trajectory $\mathcal{P}_i$, and any robot $j$
constructing the masked zone $\Pi_{\mathcal{P}_i}|_{\mathcal{G}_j}$
using only its private local map $\mathcal{G}_j$, we have
\begin{equation}
  \Pi_{\mathcal{P}_i}\big|_{\mathcal{G}_j}
  \;\subseteq\;
  \Pi_{\mathcal{P}_i}\big|_{\mathcal{G}_\text{cent}}
  \cap \mathcal{G}_j.
  \label{eq:mask_conservative}
\end{equation}
\end{proposition}
\vspace{0.25cm}
\begin{proof}
Since $\mathcal{G}_j \subseteq \mathcal{G}_\text{cent}$, the minimum in~\eqref{eq:risk_min_avar}
computed over $\mathcal{G}_j$ as the smaller set gives $\alpha^j(p)>\alpha^\text{cent}(p)$ for every $p$. Since the masking radius 
in~\eqref{eq:masking_rad} is decreasing in $\alpha$, it follows that $r^{j}_\text{mask}(p) \leq r^\text{cent}_\text{mask}(p)$, and hence each local waypoint ball is contained in its centralized counterpart.
\end{proof}

Combining Propositions~\ref{prop:localsum}, \ref{prop:mask} we now define the distributed EIG of
robot $i$ as the sum of its information contributions across all
local maps,
\begin{equation}
  \mathcal{I}^\text{dist}_i
  \;:=\;
  \sum_{j=1}^{N}
  \mathcal{I}\!\left(T_i;\;
    \Pi_{\mathcal{P}_i}\big|_{\mathcal{G}_j};\;
    \mathcal{G}_j\right).
  \label{eq:Idist}
\end{equation}

\begin{lemma}[Centralized EIG lower bound]
\label{le:lower_bound}
For each robot $i \in \{1, \ldots, N\}$, the distributed EIG is a lower bound of the centralized EIG
\begin{equation}
  \mathcal{I}^\text{dist}_i
  \;\leq\;
  \mathcal{I}\!\left(T_i;\;
    \Pi_{\mathcal{P}_i}\big|_{\mathcal{G}_\text{cent}};\;
    \mathcal{G}_\text{cent}\right).
  \label{eq:per_robot_bound}
\end{equation}
\end{lemma}
\vspace{0.25cm}
\begin{proof}
    By Proposition~2, we have
$\Pi_{\mathcal{P}_i}|_{\mathcal{G}_j} \subseteq
\Pi_{\mathcal{P}_i}|_{\mathcal{G}_\text{cent}} \cap \mathcal{G}_j$.
Under the diagonal Fisher information approximation, the EIG over a set of splats is nonnegative and monotone in the splat set, so
$\mathcal{I}(T_i;\;\Pi_{\mathcal{P}_i}|_{\mathcal{G}_j};\;\mathcal{G}_j)
\leq
\mathcal{I}(T_i;\;\Pi_{\mathcal{P}_i}|_{\mathcal{G}_\text{cent}};\;\mathcal{G}_j)$.
Summing over $j$ and applying Proposition~1
gives~\eqref{eq:per_robot_bound}.
\end{proof}

\subsection{Distributed Information Exchange}
We consider a team of $N$ robots that exchange information with each other over a connected communication graph. Each robot $i$ shares its candidate viewpoint $T_i$ and the planned trajectory
$\mathcal{P}_i$ with the team. Given the shared trajectory, robot $j$ constructs a risk-aware masked zone in its local map $\mathcal{G}_j$, denoted $\Pi_{\mathcal{P}_i}|_{\mathcal{G}_j}$.
This masked region represents the subset of Gaussian splats in $\mathcal{G}_j$ along the trajectory $\mathcal{P}_i$. Using this masked region, robot $j$ evaluates the expected information gain
(EIG) of robot $i$'s candidate viewpoint with respect to its local map as
\begin{equation}
  \mathcal{I}\!\left(T_i;\;
    \Pi_{\mathcal{P}_i}\big|_{\mathcal{G}_j};\;
    \mathcal{G}_j\right),
  \label{eq:Iij}
\end{equation}
which returns the scalar EIG and its gradient with respect to $T_i$ to robot $i$. No raw observations, Gaussian primitives, or 3DGS parameters are transmitted.
Therefore, by Lemma~\ref{le:lower_bound}, the distributed multi-robot NBV objective
can be written as
\begin{equation}
  \{T_i^{*}\}_{i=1}^{N}
  \;=\;
  \underset{\{T_i\}_{i=1}^{N}}{\arg\textrm{maximize}}\;
  \;\sum_{i=1}^{N} \sum_{j=1}^{N}
  \mathcal{I}\!\left(T_i;\;
    \Pi_{\mathcal{P}_i}\big|_{\mathcal{G}_j};\;
    \mathcal{G}_j\right).
  \label{eq:dist_obj}
\end{equation}
The objective in~\eqref{eq:dist_obj} maximizes the total EIG across
the robot team by allowing each robot to evaluate its candidate
viewpoint using both its own map and the maps of other robots, without sharing their local maps.

\subsection{Distributed Optimization using ADMM}
To solve the distributed NBV optimization problem
in~\eqref{eq:dist_obj}, we employ the Consensus Alternating Direction Method of Multipliers (C-ADMM)~\cite{boyd2011admm}. Let
\begin{equation*}
  T^{(i)}
  \;:=\;
  \big(T^{(i)}_1, T^{(i)}_2, \ldots, T^{(i)}_N\big)
\end{equation*}
denote robot $i$'s local copy of the team viewpoint vector, where $T^{(i)}_k$ is robot $i$'s copy of robot $k$'s viewpoint. Each robot's contribution to the team objective is the sum of EIG terms
evaluated on its own local map,
\begin{equation*}
  F_i(\T^{(i)})
  \;:=\;
  \sum_{k=1}^{N}
  \mathcal{I}\!\left(T^{(i)}_k;\;
    \Pi_{\mathcal{P}_k}\big|_{\mathcal{G}_i};\;
    \mathcal{G}_i\right),
\end{equation*}
which robot $i$ can evaluate locally using $\mathcal{G}_i$ and the trajectories $\{\mathcal{P}_k\}_{k=1}^{N}$ received from
the team.

The distributed optimization problem is then derived as
\begin{equation}
  \begin{aligned}
    \{T_i^{*}\}_{i=1}^{N}
    \;=\; & \argmax_{\{T^{(i)}\}}
      \;\sum_{i=1}^{N} F_i\!\big(T^{(i)}\big) \\
    \text{s.t.}\;\; & T^{(i)} = T^{(j)},
      \quad \forall (i,j) \in \mathcal{E},
  \end{aligned}
  \label{eq:cadmm_problem}
\end{equation}
where the edge-wise consensus constraints enforce agreement between the local copies of neighboring robots. For each communication edge $(i,j) \in \mathcal{E}$, we introduce an auxiliary consensus variable $z_{ij}$ and a dual variable $\lambda_{ij}$.

The augmented Lagrangian associated
with~\eqref{eq:cadmm_problem} is formulated as
\begin{multline}
\mathcal{L}_a
= \sum_{i=1}^{N} \Big[-F_i\big(T^{(i)}\big)\Big] \\
+ \sum_{i=1}^{N} \sum_{j \in \mathcal{N}_i}
    \!\Big[
      \lambda_{ij}^{\top}\big(T^{(i)} - z_{ij}\big)
      + \tfrac{\rho}{2}\big\|T^{(i)} - z_{ij}\big\|_2^{2}
    \Big],
  \label{eq:aug_lag}
\end{multline}
where $\rho > 0$ is the penalty parameter and $\|\cdot\|_2$ denotes the Euclidean norm.

\textit{Remark 1:} Viewpoint $T_i \in SE(3)$ has its translational component fixed to the trajectory waypoint, and its orientation restricted to a yaw rotation $\psi_i \in \mathbb{S}^1$ within $[-\pi,\pi]$~\cite{Khass2025_NBV_RiskAverse}.
The team viewpoint vector $T^{(i)} = (\psi^{(i)}_1,\dots,\psi^{(i)}_N)$
is therefore identified with a vector in $\mathbb{R}^N$, so the inner product, subtraction, and arithmetic average in~\eqref{eq:aug_lag}--\eqref{eq:cadmm_dual}
are well-defined componentwise. We initialize $\lambda_{ij}^0 = 0$ for all $(i,j) \in \mathcal{E}$, which together with the symmetric updates maintains $\lambda_{ij}^t + \lambda_{ji}^t = 0$ throughout, so the
$z$-update reduces to the arithmetic mean of the endpoint copies. For full $SE(3)$ viewpoint optimization, the updates generalize by replacing the Euclidean operations with their Lie-algebra
counterparts on $\mathfrak{se}(3) \cong \mathbb{R}^6$.

The C-ADMM updates at iteration $t$ are given by
\begin{align}
  T^{(i),t+1}
  \;&=\; \argmin_{T^{(i)}}\;
    \bigg[\,
      {-}F_i\!\big(T^{(i)}\big)
      + \sum_{j \in \mathcal{N}_i}
        \big(\lambda^{t}_{ij}\big)^{\!\top}\!
        \big(T^{(i)} - z^{t}_{ij}\big)
        \nonumber\\
    & \qquad\quad
      + \tfrac{\rho}{2}
        \sum_{j \in \mathcal{N}_i}
        \big\|T^{(i)} - z^{t}_{ij}\big\|_2^{2}
    \bigg],
  \label{eq:cadmm_primal}\\[2pt]
  z^{t+1}_{ij}
  \;&=\;
  \tfrac{1}{2}\!\left(T^{(i),t+1} + T^{(j),t+1}\right),
  \quad \forall (i,j) \in \mathcal{E},
  \label{eq:cadmm_z}\\[2pt]
  \lambda^{t+1}_{ij}
  \;&=\;
  \lambda^{t}_{ij}
  + \rho\!\left(T^{(i),t+1} - z^{t+1}_{ij}\right),
  \quad \forall (i,j) \in \mathcal{E},
  \label{eq:cadmm_dual}
\end{align}
where $\mathcal{N}_i$ denotes the one-hop neighborhood of robot $i$
in the communication graph. The primal update
in~\eqref{eq:cadmm_primal} is solved using stochastic gradient descent on the local viewpoint parameters. The consensus and dual updates in~\eqref{eq:cadmm_z}--\eqref{eq:cadmm_dual} require only exchange of local copies along $\mathcal{E}$, so the iteration communication payload is independent of the size of the 3DGS maps.
\begin{algorithm}[t!]
\caption{Distributed Risk-Aware Multi-Agent NBV Selection}
\label{alg:distributed_nbv}
\begin{algorithmic}[1]
\STATE \textbf{Input:} Team of robots $N$, local maps $\{\mathcal{G}_i\}_{i=1}^N$, planned trajectories $\{\mathcal{P}_i\}_{i=1}^N$, communication graph $(\mathcal{V},\mathcal{E})$
\STATE \textbf{Output:} $\{T_i^\star\}_{i=1}^N$
\STATE Initialize  $\{T^{(i),0}\}_{i=1}^N$,  $\{z_{ij}^{0}\}_{(i,j)\in\mathcal{E}}$, $\{\lambda_{ij}^{0}\}_{(i,j)\in\mathcal{E}}$
\FORALL{$i \in \{1,\dots,N\}$}
    \STATE Share $(T_i,\mathcal{P}_i)$ with neighbors $\mathcal{N}_i$
    \STATE Construct $\Pi_{\mathcal{P}_k}\big|_{\mathcal{G}_i}$ for $k \in \mathcal{N}_i$
\ENDFOR
\FOR{$t=0$ to $T_{\mathrm{ADMM}}-1$}
    \FORALL{$i \in \{1,\dots,N\}$}
        \STATE Update local copy $T^{(i),t+1}$ via \eqref{eq:cadmm_primal}
        \STATE Exchange $T^{(i),t+1}$ with neighbors $j \in \mathcal{N}_i$
    \ENDFOR
    \FORALL{$(i,j) \in \mathcal{E}$}
        \STATE Update edge consensus $z_{ij}^{t+1}$ via \eqref{eq:cadmm_z}
        \STATE Update edge dual $\lambda_{ij}^{t+1}$ via \eqref{eq:cadmm_dual}
    \ENDFOR
\ENDFOR
\FORALL{$i \in \{1,\dots,N\}$}
    \STATE $T_i^\star \leftarrow \big[T^{(i),T_{\mathrm{ADMM}}}\big]_{i}$
    \STATE Execute $T_i^\star$ and update $\mathcal{G}_i$
\ENDFOR
\end{algorithmic}
\end{algorithm}

\begin{table*}[t]
\vspace{0.30cm}
\centering
\renewcommand{\arraystretch}{1}
\setlength{\tabcolsep}{5pt}
\begin{tabular}{l|ccc|ccc|ccc|ccc}
\toprule
\textbf{Algorithm}
& \multicolumn{3}{c|}{\textbf{Denmark}}
& \multicolumn{3}{c|}{\textbf{Cantwell}}
& \multicolumn{3}{c|}{\textbf{Ribera}}
& \multicolumn{3}{c}{\textbf{Swormville}} \\
\cmidrule(lr){2-4} \cmidrule(lr){5-7} \cmidrule(lr){8-10} \cmidrule(lr){11-13}
& AV@R$\uparrow$ & PSNR$\uparrow$ & Depth$\downarrow$
& AV@R$\uparrow$ & PSNR$\uparrow$ & Depth$\downarrow$
& AV@R$\uparrow$ & PSNR$\uparrow$ & Depth$\downarrow$
& AV@R$\uparrow$ & PSNR$\uparrow$ & Depth$\downarrow$ \\
\midrule
Centralized
& 0.85 & 19.57 & 0.15
& 0.29 & 17.68 & 0.25
& 0.82 & 18.13 & 0.28
& 0.63 & 20.55 & 0.15\\
Distributed
& 0.72 & 18.85 & 0.18
& 0.22 & 16.98 & 0.29
& 0.75 & 17.09 & 0.36
& 0.61 & 20.20 & 0.15 \\
Limited Sharing
& 0.80 & 19.21 & 0.16
& 0.28 & 17.61 & 0.26
& 0.76 & 17.39 & 0.35
& 0.61 & 20.53 & 0.15 \\
Single Agent
& 0.67 & 16.75 & 0.31
& 0.18 & 15.41 & 0.35
& 0.57 & 14.65 & 0.42
& 0.57 & 16.45 & 0.41 \\
\bottomrule
\end{tabular}
\caption{
Performance comparison of the proposed algorithms and the single-agent method across Gibson environments.
We report AV@R (risk), PSNR, and mean depth error calculated at the robot's trajectory waypoints for different viewing angles, averaged across five agents. 
Lower depth error is better, while higher AV@R and PSNR indicate improved performance.
}
\label{tab:gibson_env_algorithm_nomask}
\vspace{-0.35cm}
\end{table*}
\begin{figure*}[t]
    \centering
\includegraphics[width=\textwidth,height=0.50\textheight,keepaspectratio,
    trim={1cm 11.5cm 2.1cm 0.5cm},clip]{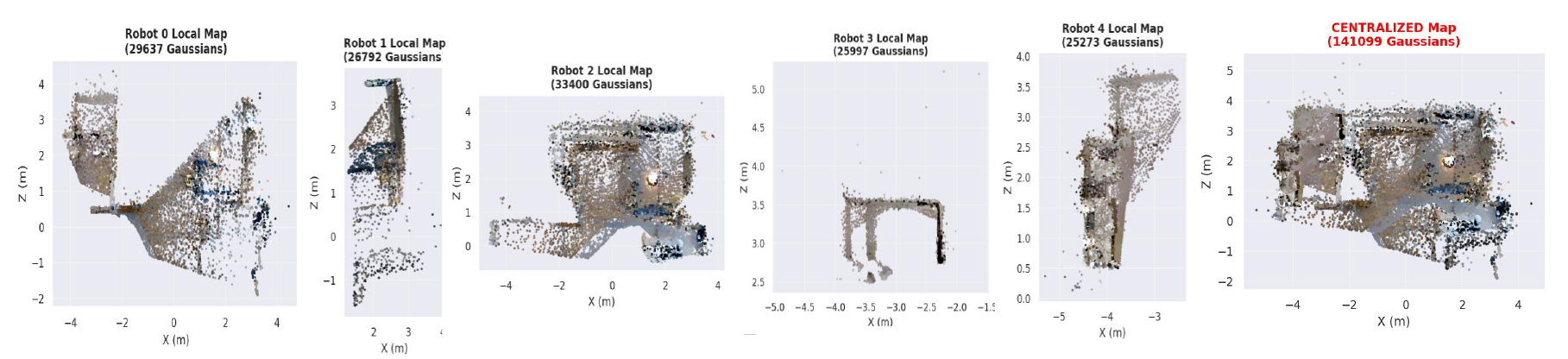}
    \caption{Centralized map sharing with five agents where with small local maps, the global map and sharing size grow rapidly.}
\end{figure*}

\section{Experiments}
\label{sec:experiments}
We evaluate the proposed multi-agent NBV optimization in three settings: (i) a centralized baseline with a shared global map, (ii) a distributed viewpoint-suggestion method with private local 3DGS maps, and (iii) a limited-sharing method that combines selective map exchange with distributed viewpoint suggestions.
\subsection{Implementation Pipeline}
To navigate toward the goal in unknown environments, each agent adopts a frontier-based subgoal selection strategy. At each planning cycle, frontiers are extracted from the agent's updated map and filtered based on their size and proximity to the goal to form a candidate subgoal set. The agent then selects the most informative subgoal and computes a safe future trajectory using a risk-averse $A^\star$ planner that jointly accounts for path length and collision risk, yielding a sequence of waypoints $\mathcal{P}=\{p_1,p_2,\ldots,p_K\}$. The proposed NBV framework is subsequently applied along $\mathcal{P}$ to reduce uncertainty in trajectory-relevant regions while maintaining safety. This procedure is repeated in a receding horizon until the goal is reached; the pipeline is summarized in Algorithm~\ref{alg:distributed_nbv}. We also evaluate a limited-sharing variant, in which robots exchange only sparse, trajectory-relevant subsets of their local 3DGS maps selected from team-suggested informative viewpoints. The updated local maps are then used in the distributed risk-aware NBV optimization.
\begin{table}[t!]
\centering
\renewcommand{\arraystretch}{1}
\setlength{\tabcolsep}{9pt}
\small
\begin{tabular}{l|cc}
\toprule
\textbf{Algorithms}
& \multicolumn{2}{c}{\textbf{Communication size Per Robot}} \\
\cmidrule(lr){2-3}
& \makecell{Received}
& \makecell{Sent} \\
\midrule
Distributed
& 0.028 KB & 0.206 KB \\
Limited Sharing
& 25 KB   & 25 KB   \\
\bottomrule
\end{tabular}
\caption{
Per-robot communication size for the Distributed and limited sharing methods.
}
\vspace{-0.35cm}
\label{tab:comm_single_agent}
\end{table}
\subsection{Experimental Setup}
\label{subsec:exp_setup}
We use the Habitat simulator~\cite{szot2021habitat} in indoor Gibson environments~\cite{xia2018gibson}. In each experiment, agents are initially located at the Habitat default positions in an empty map with no prior information. We assume each agent has an attached RGB-D camera. 
At each frame, the agent acquires a synchronized color and depth image with size $800 \times 800$ pixels. The camera field of view is set to $90^\circ$ vertically and horizontally. 
Agents operate in a discrete action space consisting of:
(i) \textsc{Move Forward} by $5$\,cm,
(ii) \textsc{Turn Left} by $5^\circ$, and
(iii) \textsc{Turn Right} by $5^\circ$ to
follow a risk-averse optimized path planned via the $A^\star$ algorithm at each planning step. In this experiment, a path is planned on a 2D occupancy grid with a 5cm resolution. Rather than evaluating a discrete set of candidate viewpoints, the agent at each pose optimizes its viewing direction using a gradient-based method to find the optimal pose.
Each agent updates its map incrementally as new observations are collected. Experiments are implemented on a workstation equipped with an Intel Core i9-13900K CPU and NVIDIA
RTX~A2000 GPU.
\begin{table}[t!]
\centering
\renewcommand{\arraystretch}{0.9}
\setlength{\tabcolsep}{8pt}
\small
\begin{tabular}{l|ccc}
\toprule
\textbf{Algorithm}
& \multicolumn{3}{c}{\textbf{2 Agents Communication Size}} \\
\cmidrule(lr){2-4}
& \makecell{Agent 1 \\ }
& \makecell{Agent 2 \\ }
& \makecell{Shared Data size \\ } \\
\midrule
\multirow{4}{*}{Centralized}
& 34{,}911  & 26{,}800  & 3.26 MB  \\
& 62{,}634  & 60{,}898  & 6.59 MB  \\
& 161{,}396 & 148{,}005 & 16.50 MB \\
& 190{,}716 & 189{,}123 & 20.28 MB \\
\bottomrule
\end{tabular}
\caption{
Communication analysis of the centralized algorithm for two agents reports the local map size (splats) for each agent and the size of the shared data transmitted.
}
\vspace{-0.35cm}
\end{table}
\begin{table*}[t]
\vspace{0.35cm}
\centering
\renewcommand{\arraystretch}{1.1}
\setlength{\tabcolsep}{6pt}
\begin{tabular}{l|cc|cc|cc|cc}
\toprule
\textbf{Algorithm}
& \multicolumn{2}{c|}{\textbf{Denmark}}
& \multicolumn{2}{c|}{\textbf{Cantwell}}
& \multicolumn{2}{c|}{\textbf{Ribera}}
& \multicolumn{2}{c}{\textbf{Swormville}} \\
\cmidrule(lr){2-3}
\cmidrule(lr){4-5}
\cmidrule(lr){6-7}
\cmidrule(lr){8-9}
& Path length$\downarrow$ & AV@R$\uparrow$
& Path length$\downarrow$ & AV@R$\uparrow$
& Path length$\downarrow$ & AV@R$\uparrow$
& Path length$\downarrow$ & AV@R$\uparrow$ \\
\midrule
Centralized        
& 27.9\% & 43.1\% 
& 27.6\% & 13.3\% 
& 11.7\% & 20.7\%
& 24.7\% & 38.1\% \\
Distributed      
& 29.5\% & 37.8\% 
& 31.6\% & 12.5\%
& 14.3\% & 19.7\% 
& 28.8\% & 37.0\% \\
Limited Sharing    
& 28.2\% & 41.5\% 
& 29.8\% & 13.3\%  
& 12.1\% & 20.1\% 
& 25.3\% & 37.5\% \\
Single Agent       
& 31.1\% & 36.8\%
& 36.2\% & 11.7\%
& 15.7\% & 17.2\% 
& 30.9\% & 35.5\% \\
\bottomrule
\end{tabular}
\vspace{2pt}
\caption{
Quantitative comparison of path length and AV@R for each algorithm across different environments. Each robot planned a safe path toward the selected frontier and optimized the yaw angle of the next pose to maximize the expected information gain. Robots replanning their safe path after updating their map belief.}
\label{tab:gibson_env_algorithm_nomask}
\end{table*}
\begin{figure*}[t]
    \centering
\includegraphics[width=\textwidth,height=0.50\textheight,keepaspectratio,
    trim={0cm 5.5cm 0.1cm 3.5cm},clip]{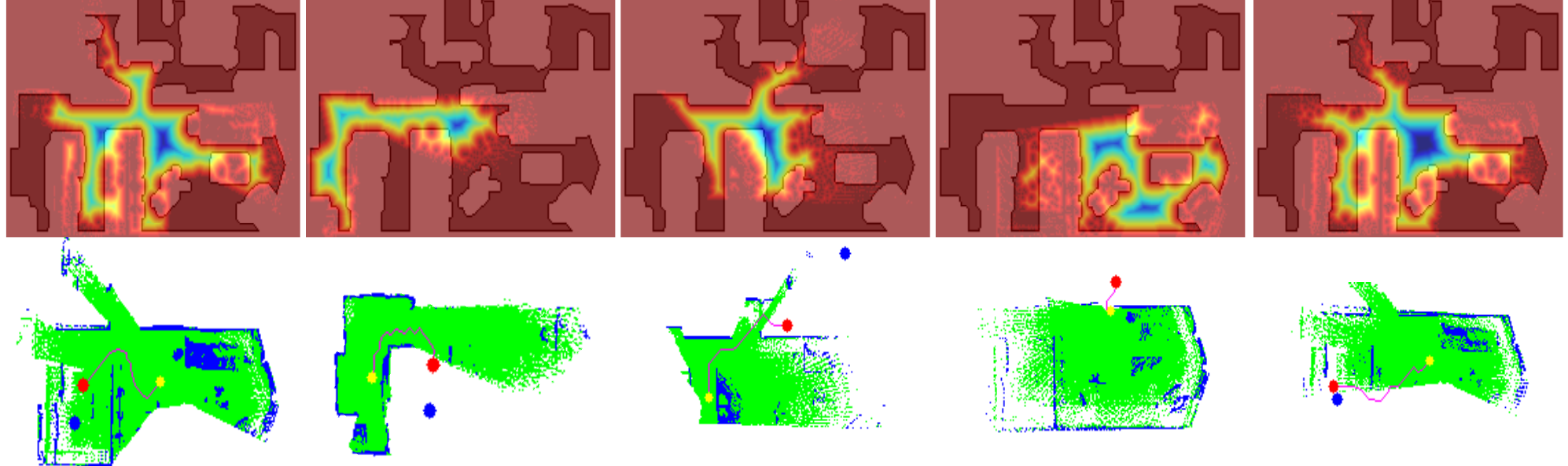}
    \caption{Illustration of five robots' local map beliefs with the limited-sharing algorithm. The top row indicates the collision risk heat map. The bottom row shows the robot's safe path to the next waypoints in each robot's local map belief.}
    \label{fig:limited_sharing}
\end{figure*}
\subsection{Evaluation Metrics}
\label{subsec:metrics}
We evaluate our method across safe planning, mapping quality, and communication efficiency.
To compare the trajectory safety level, we adopt the AV@R metric defined in section ~\ref{subsec:risk_metric}. After each algorithm execution with updated local maps, we compute a reference path using $A^\star$ from the same start/goal position and report $\text{Total Risk}_\varepsilon$ as $\frac{1}{K}\sum_{k=1}^{K}\alpha(p_k),$
where $\alpha(\cdot)$ denotes the AV@R value of the planned trajectory
$\mathcal{P}=\{p_1,\dots,p_K\}$ (see Table~\ref{tab:gibson_env_algorithm_nomask}).
\begin{table}[t]
\centering
\small
\renewcommand{\arraystretch}{1.05}
\setlength{\tabcolsep}{3pt}

\resizebox{\columnwidth}{!}{%
\begin{tabular}{lcccccc}
\toprule
\textbf{Algorithm} &
\makecell{$R_1$\\\footnotesize(Splats)} &
\makecell{$R_2$\\\footnotesize(Splats)} &
\makecell{$R_3$\\\footnotesize(Splats)} &
\makecell{$R_4$\\\footnotesize(Splats)} &
\makecell{$R_5$\\\footnotesize(Splats)} &
\makecell{\textbf{Total}\\\footnotesize(MB)} \\
\midrule
Centralized & 56{,}741 & 51{,}055 & 30{,}174 & 50{,}368 & 22{,}909 & 11.19 \\
Centralized & 64{,}612 & 55{,}718 & 34{,}548 & 61{,}713 & 37{,}310 & 13.55 \\
\bottomrule
\end{tabular}%
}
\caption{Communication size for five agents with a centralized algorithm reports the number of splats in local 3DGS maps and the total data exchanged size.}
\vspace{-0.35cm}
\label{tab:comm_5agents_centralized}
\end{table}

\begin{table}[t!]
\centering
\renewcommand{\arraystretch}{1.1}
\setlength{\tabcolsep}{15pt}
\begin{tabular}{l|cc}
\toprule
\textbf{Team size} & \multicolumn{2}{c}{\textbf{Shared data size}} \\
\cmidrule(lr){2-3}
& \textbf{Centralized} & \textbf{Limited Sharing} \\
\midrule
2 Agents  & 3.63 MB  & 0.050 MB \\
5 Agents  & 8.95 MB  & 0.50 MB  \\
7 Agents  & 10.39 MB & 1.050 MB \\
10 Agents & 16.84 MB & 2.150 MB \\
\bottomrule
\end{tabular}
\caption{Comparison of shared data sizes across varying team sizes using centralized and limited-sharing algorithms during the same planning step in the Cantwell scene (Gibson dataset).
 }
\label{tab:shared_data_size}
\vspace{-0.45cm}
\end{table}
In addition, we report PSNR (Peak Signal-to-Noise Ratio)~\cite{HoreZiou2010PSNRSSIM} and depth MAE (Mean Absolute Error) to evaluate mapping quality by rendering views at each pose along the trajectory with different viewing angles.
Moreover, to quantify communication efficiency and scalability, we report the number of bytes transmitted per map share, based on the size of the 3DGS map representation. This metric characterizes each method's communication load and highlights the cost-benefit trade-off among algorithms.
\begin{figure}[t]
    \centering
\includegraphics[width=1.15\linewidth,height=1.2\textheight,keepaspectratio, trim={2.9cm 0.7cm 0.7cm 0.5cm},
    clip]{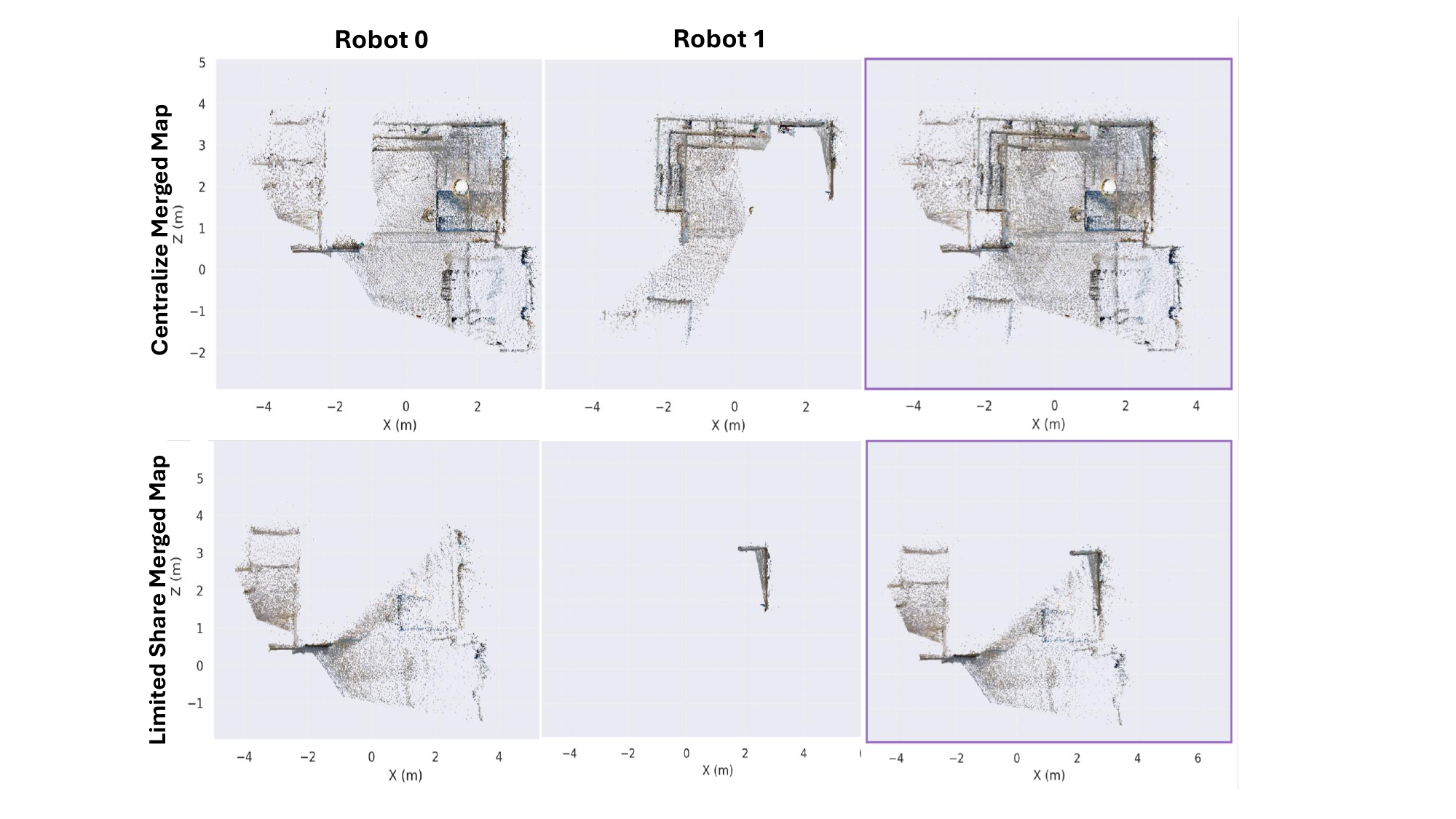}
    \caption{
    Comparison of local 3DGS maps sharing for two agents with centralized and limited-sharing algorithms.}
    \label{fig:limited_sharing}
    \vspace{-0.5 cm}
\end{figure}
\subsection{Performance and Communication Analysis}
\label{subsec:results}
We compare the centralized, distributed, and limited-sharing methods. In the centralized approach, robots share their local 3DGS maps to construct a global map. As shown in Tables~II, III, V, and VI, this strategy incurs communication costs that grow rapidly with map size, even for relatively small local maps. In practice, 3DGS maps may contain $10^6$--$10^7$ splats, resulting in communication payloads on the order of tens to hundreds of megabytes.
In contrast, the communication load of the distributed method is independent of local map size because robots exchange only compact viewpoint-related information. Under an ideal 64-bit floating-point representation, the communication payload is 136 bytes per exchange. In our implementation, the exchanged data are encoded in JSON format.

The limited-sharing method bridges these two extremes by initially sharing a robot's locally optimal NBV with a bounded set of Gaussian splats selected from the corresponding view cone, followed by pose exchange. We cap the number of shared splats at 400, which makes the communication cost effectively independent of the full local map size, as shown in Table~II.

Table~IV compares the performance of the algorithm across multiple Gibson environments. For each method, a reference trajectory is planned using $A^\star$ from the same start and goal positions in the local map. The multi-agent methods consistently outperform the single-agent baseline in both risk and reconstruction quality. Although the mapping metrics do not exceed prior mapping-focused baselines~\cite{Jiang2024_AGSLAM}, this is expected for two reasons: first, those methods evaluate mapping performance over much longer horizons, whereas our experiments focus on shorter task-driven trajectories; second, our objective is not full-scene reconstruction, but uncertainty reduction in trajectory-relevant regions to support safer and more efficient navigation.

\section{Conclusion and Limitations}
We proposed Multi-Agent Next-Best-View Optimization for Risk-Averse Planning framework based on 3D Gaussian splatting maps. Optimizing view point to reduce uncertainty in regions with higher risk of collision along each robot's trajectory enables safe path planning under limited communication. Simulation results show that distributed, limited-sharing algorithms reduce collision risk and trajectory-relevant map uncertainty while improving scalability.
The current study is limited to simulation and does not yet model real-world sensing. In addition, the proposed framework targets trajectory-relevant uncertainty reduction by design, which is the regime relevant to safe navigation. It assumes that each robot has onboard communication and computation capabilities to perform the required coordination.

{\footnotesize
\bibliographystyle{IEEEtran} 
\bibliography{refs}
}
\addtolength{\textheight}{-12cm}   

\end{document}